\documentclass{elsarticle}

\biboptions{sort&compress}

\usepackage{url}
\usepackage{bm}
\usepackage{balance}

\usepackage[utf8]{inputenc}
\usepackage[T1]{fontenc}

\usepackage[babel = true]{microtype}

\usepackage[american]{babel}
\usepackage{csquotes}

\usepackage[dvipsnames]{xcolor}
\definecolor{stroke1}{HTML}{2574A9}

\usepackage{xspace}
\usepackage{subcaption}
\usepackage{booktabs}

\usepackage{amsmath}
\usepackage{amssymb}
\usepackage{amsthm}
\usepackage{thmtools}
\usepackage{mathtools}
\usepackage{dsfont}
\usepackage{enumitem}

\newtheorem{theorem}{Theorem}

\usepackage
[
    ruled,
    vlined,
    linesnumbered,
]
{algorithm2e}

\usepackage{todonotes}

\usepackage{balance}

\usepackage
[
    bookmarks = true,
    bookmarksopen = false,
    bookmarksnumbered = true,
    pdfstartpage = 1,
    colorlinks = true,
    allcolors = stroke1,
]
{hyperref}

\usepackage
[
    noabbrev,
    nameinlink,
]
{cleveref}

\newcommand*{\ignore}[1]{}

\let\originalleft\left
\let\originalright\right
\renewcommand{\left}{\mathopen{}\mathclose\bgroup\originalleft}
\renewcommand{\right}{\aftergroup\egroup\originalright}

\newcommand*{\N}{\mathds{N}}
\newcommand*{\R}{\mathds{R}}
\newcommand*{\OM}{\textsc{OneMax}\xspace}

\newcommand*{\E}{\mathrm{E}}

\newcommand*{\UBM}{\textsc{EqualBlocksOneMax}\xspace}
\newcommand*{\ubm}{\textsc{EBOM}\xspace}

\newcommand*{\hammingDistance}[2]{\mathrm{H}\left(#1, #2\right)}

\begin{document}

\begin{frontmatter}
    \title{Bivariate Estimation-of-Distribution Algorithms Can Find an Exponential Number of Optima}

    \author{Benjamin Doerr}
    %\ead{doerr@lix.polytechnique.fr}

    \author{Martin~S. Krejca\corref{cor1}}
    \ead{martin.krejca@polytechnique.edu}

    \cortext[cor1]{%
        Corresponding author\\
        Published version available here: \url{https://doi.org/10.1016/j.tcs.2023.114074}
    }

    \address{Laboratoire d'Informatique (LIX), CNRS, Ecole Polytechnique, Institut Polytechnique de Paris, Palaiseau, France}

    \begin{abstract}
        Finding a large set of optima in a multimodal optimization landscape is a challenging task. Classical population-based evolutionary algorithms typically converge only to a single solution. While this can be counteracted by applying niching strategies, the number of optima is nonetheless trivially bounded by the population size.
        Estimation-of-distribution algorithms (EDAs) are an alternative, maintaining a probabilistic model of the solution space instead of a population. Such a model is able to implicitly represent a solution set far larger than any realistic population size.

        To support the study of how optimization algorithms handle large sets of optima, we propose the test function \UBM (\ubm). It has an easy fitness landscape with exponentially many optima. We show that the bivariate EDA \emph{mutual-information-maximizing input clustering}, without any problem-specific modification, quickly generates a model that behaves very similarly to a theoretically ideal model for \ubm, which samples each of the exponentially many optima with the same maximal probability.
        We also prove via mathematical means that no univariate model can come close to having this property: If the probability to sample an optimum is at least inverse-polynomial, there is a Hamming ball of logarithmic radius such that, with high probability, each sample is in this ball.
    \end{abstract}

\end{frontmatter}

\section{Introduction}
\label{sec:introduction}

A key feature of evolutionary algorithms (EAs) is their applicability to a wide range of optimization problems. EAs require little problem-specific knowledge and generally provide the user with a good solution. Since many real-world optimization problems are multimodal \cite{BeldaMTLG07MulitmodalSearch,HocaogluS97MultimodalOptimization,SinghD06NichingMethods}, it is desirable for an EA to return multiple solutions. This way, the user also gains precious insight into their problem.

Unfortunately, classical population-based EAs tend to converge to a single solution, due to strong selection operators and due to a long-known phenomenon called \emph{genetic drift}~\cite{DeJong75GeneticDrift}. In order to counteract this behavior, different techniques have been introduced, commonly subsumed under the term \emph{niching}~\cite{SinghD06NichingMethods,Mahfoud96Niching,MillerS96Niching}. These techniques maintain diversity in the population and assist in finding and keeping multiple good solutions. While this approach is useful for increasing the number of different solutions, it still limits the insights gained about the underlying problem, as the only information the EA returns is the solutions themselves. As such, it only provides information about areas of the search space that it has visited and does not propose further promising regions.

A different algorithmic approach that aims to additionally incorporate information about the entire search space is the framework of \emph{estimation-of-distribution algorithms} (EDAs; \cite{PelikanHL15SurveyOnEDAs}). Instead of an explicit set of solutions, EDAs maintain a probabilistic model of the search space. This model acts as a solution-generating mechanism and reflects information about which parts of the search space seem more favorable than others. An EDA evolves its model based on samples drawn from it. This way, the model is refined such that it generates better solutions with higher probability. In the end, an EDA returns the best solutions found as well as its model.

EDAs are commonly classified by the power of their model~\cite{LarranagaL02EDAs,PelikanHL15SurveyOnEDAs}. This results in the following trade-off: an EDA with a simple model performs an update quickly but may be badly suited to accurately represent the distribution of good solutions. In contrast, the update of an EDA with a complex model is computationally expensive, but the model is better capable of representing good solutions. The complexity of a model is determined by how many dependencies it can detect among different problem variables. For example, a \emph{univariate} EDA assumes independence of all problem variables, whereas a \emph{bivariate} EDA can represent dependencies among pairs of variables. We go into detail about these types of EDAs in \Cref{sec:EDAsProbabilisticModels}.

While increasing the complexity of an EDA's model is useful for finding optima in a larger class of problems~\cite{PelikanG03hBOAperformsWell}, it is not evident that an increased model complexity is also useful for finding \emph{multiple optima} or representing them adequately in the model. In fact, EDAs have been designed specifically with the intention of being used for multimodal optimization. Pe\~{n}a et~al.~\cite{PenaLL05EDAsFindingMultipleOptima} introduce the \emph{unsupervised estimation of Bayesian network algorithm} (UEBNA), which uses unsupervised learning in order to generate the Bayesian network of its model. The algorithm is tested against other EDAs and evaluated (mostly) on bisection problems on graphs with many symmetries that only have a low number of optima (two to six). Interestingly, for the larger problems, even UEBNA is not able to find all optima. Thus, the test functions seem to be hard, and the experiments do not only show how many optima the algorithm can find but also how well it copes with hard problems.

A similar setting has been considered by Chuang and Hsu~\cite{ChuangH10MultimodalEDA}, who introduce an EDA that is also specifically tailored toward multimodal optimization. However, they evaluate their results only on trap functions with a low number of optima (two to four). Thus, the focus of their work is arguably also more on the hardness of the problem than on finding many optima.

Hauschild et~al.~\cite{HauschildPLS07hBOAprobabilisticModelAnalysis} consider the \emph{hierarchical Bayesian optimization algorithm} and analyze how well its model reflects the problem structure of two hard test functions. They show that the structure is best reflected during the middle of the run and that it is then simplified toward the end. This makes sense, as the model aims to reflect best how to generate \emph{optimal} solutions. This does not need to coincide with how the \emph{entire} structure can be reflected. For example, if the problem has a single solution, it suffices to have a simple model that only generates this solution in a straightforward way. Again, the focus of the authors is rather based on the hardness of the problem (its structure) instead of representing many optima.

Overall, to the best of our knowledge, all results dedicated to finding many different optima consider functions for which finding an optimum at all is already a challenge.
We note though that Echegoyen et~al.~\cite{EchegoyenMSL12AnalyzingProbabilisticModels} thoroughly analyze how the probabilistic model of the EDA called \emph{estimation of Bayesian network algorithm} (EBNA~\cite{EtxeberriaL99EBNA}) evolves on unimodal and on hard problems, such as traps or \textsc{MaxSat}---a task related to what we set out to analyze in this article.

We introduce the test function \UBM (\ubm, \Cref{sec:UBM}), which has an exponential number of optimal solutions. It is easy in the sense that all local optima are also global optima. We are interested in how well the underlying structure of the optima can be detected by an algorithm.\footnote{Our code is available on GitHub~\cite{OurCode}.}

Univariate EDAs apparently are not suitable to return a model that represents the exponential number of optima of \ubm. We make this intuitive statement precise in \Cref{sec:theory}, where we show that any univariate model with an at least moderate probability to sample an optimum of \ubm has the property that its samples are highly concentrated in a Hamming ball of only logarithmic radius (\Cref{thm:number_of_optima_univariate}).

%EAs can find at most as many different optima as the size of their population. Depending on the actual size, this may be a large number but will be polynomial for any reasonable run time of the algorithm. This is still insignificant to the total number of optima of \ubm.
%\merk{I removed this sentence. Now that we prove something about univariate EDAs, the question could arise if for classic EAs a small population could not nevertheless constitute a rich probabilistic model. E.g., if we had $k$ random optima in the population, would this not tell us a lot about EBOM?}

In contrast, we show that \emph{mutual-information-maximizing input clustering} (MIMIC; \cite{BonetIV96MIMIC}), arguably the simplest bivariate EDA, represents the structure of \ubm well. It builds a model that behaves very similarly to an ideal model for \ubm, which creates all optimal solutions with the maximal probability possible. Our experiments (\Cref{sec:experiments}) show that, for almost all input sizes we consider, MIMIC samples about $1 \cdot 10^4$ to $4.5 \cdot 10^4$ optima per run and never samples an optimum twice. As \ubm can be described by a bivariate model, our results suggest that bivariate EDAs are well suited to reasonably capture the set of all optima for functions they can optimize.

Following, we discuss different types of EDAs in order to explain how common probabilistic models look like and why univariate models are unsuited for representing multiple optima. In \Cref{sec:preliminaries}, we present MIMIC, the definition of \ubm, and what an ideal model for \ubm is. In \Cref{sec:experiments}, we explain our test setup and discuss our results. We complement our findings in \Cref{sec:theory}, where we show that univariate EDAs are not capable of representing such a large number of optima. We conclude our paper in \Cref{sec:conclusion}.
This paper extends our prior work~\cite{DoerrK20MIMICforGECCO} via the results from \Cref{sec:theory}.

\subsection{Types of EDAs}
\label{sec:EDAsProbabilisticModels}

A common way of classifying EDAs is with respect to how they decompose a problem~\cite{PelikanHL15SurveyOnEDAs}. Such a decomposition is typically based on representing a probability distribution over the search space as a product of various probabilities that may share dependencies.
One possible and commonly applied way of doing so dates back to Henrion~\cite{Henrion1988ProbabilisticLogicSampling}, who suggested to store this information compactly by a probabilistic graphical model (PGM; \cite{KollerF09GraphicalModels}). As the name suggests, PGMs use graphs for storing information about probability distributions, where variables are represented as nodes and dependencies as edges. The arguably best known type of PGM are Bayesian networks (BNs).

A BN can be represented as a directed acyclic graph.
A directed edge from~$x$ to~$y$ represents that~$y$ is dependent on (at least)~$x$.
Consequently, two nodes that cannot reach each other mutually are conditionally independent with respect to their common predecessors.
Further, each node stores a probability distribution conditional on the outcomes of all of its predecessors.
This is usually done via a probability table for each node, where the probability for each possible value is stored for each possible combination of the outcomes of the predecessors.

A solution according to the probability distribution of a BN can be sampled by traversing the graph in a topological order, always determining the outcome of an input variable based on the outcome of its predecessors, utilizing the probability tables. The larger the in-degree of a node in a BN can become, the more costly it is to represent the model, as the conditional probability distribution for each node can grow quite large. Thus, the number of dependencies in the models of EDAs are usually restricted.

\subsubsection{Univariate EDAs}
The BN of a univariate EDA is an independent set. That is, each node represents a probability distribution based solely on a single variable. Hence the name \emph{uni}-variate. Examples of univariate EDAs are the \emph{compact genetic algorithm}~\cite{HarikLG99cGA} and the \emph{univariate marginal distribution algorithm}~\cite{MuehlenbeinP96UMDA}.

When optimizing functions over bit strings, the probability of each binary input variable tends to either~$0$ or~$1$ rather quickly~\cite{FriedrichKK16BalancedStable,DoerrZ20GeneticDrift}, forcing the model to put its probability mass onto a single solution. Thus, univariate EDAs are ill-suited to represent multiple solutions at once. For more theoretical investigations on this topic, please refer to a recent survey by Krejca and Witt~\cite{KrejcaW20EDAsurvey}.

\subsubsection{Bivariate EDAs}
In a bivariate EDA, each problem variable can be dependent on at most one other variable. Examples of bivariate EDAs are \emph{mutual-information-maximizing input clustering}~\cite{BonetIV96MIMIC} and the \emph{bivariate marginal distribution algorithm}~\cite{PelikanM99BMDA}.

Recently, Lehre and Nguyen~\cite{LehreN19UMDAonDeception} showed that MIMIC may have a huge advantage over univariate EDAs on deceptive functions, but this may be a consequence of a suboptimal parameter choice~\cite{DoerrK21UMDAwellOnDeception}.

Since a bivariate model can store simple dependencies, it is capable to represent multiple solutions at once. Further, the model can still be built somewhat efficiently, as there is at most a quadratic number of possible dependencies to consider when building the model. Thus, we focus on bivariate EDAs in this work.

\subsubsection{Multivariate EDAs}
This type is used as an umbrella term for any type of EDA that is able to represent some form of dependency. While the models of such EDAs can perform well on deceptive, hard functions, creating a model can be computationally expensive, as potentially many dependencies need to be checked. Examples of multivariate EDAs are the \emph{extended compact genetic algorithm}~\cite{Harik99ECGA}, the \emph{factorized distribution algorithm which learns a factorization}~\cite{MuehlenbeinM1999FDA}, the aforementioned EBNA~\cite{EtxeberriaL99EBNA}, as well as the \emph{Bayesian optimization algorithm}~\cite{PelikanGCP1999BOA} and the \emph{hierarchical Bayesian optimization algorithm}~\cite{PelikanG01hBOA}.

\section{Preliminaries}
\label{sec:preliminaries}

In this section, we introduce some notation that we use throughout the paper as well as the algorithm and the test function that we consider in our analysis in \Cref{sec:experiments}.

\subsection{Notation}

Let~$\N$ denote the set of all natural numbers, including~$0$. For $a, b \in \N$, let $[a .. b] \coloneqq [a, b] \cap \N$ denote the set of all natural numbers from~$a$ to~$b$ (including both bounds). As a special case of that notation, for $b \in \N$, let $[b] \coloneqq [1 .. b]$ denote the set of all positive natural numbers up to~$b$. For an $n \in \N$, let $\mathrm{id}_n$ denote the identity function over $[n]$.

For a logical proposition~$A$, let $\mathds{1}\{A\}$ denote the indicator function of the truth value of~$A$, that is, $\mathds{1}\{A\} = 1$ if~$A$ is true, and it is~$0$ otherwise.

We consider pseudo-Boolean optimization, that is, optimization of functions $f\colon \{0, 1\}^n \to \R$, where $n \in \N$. We call such a function \emph{fitness function.} We call a bit string $x \in \{0, 1\}^n$ an \emph{individual} and $f(x)$ the \emph{fitness of~$x$.} If not stated otherwise, let~$f$ always denote a fitness function, and let~$n$ always denote its dimension.

\subsection{Mutual-Information-Maximizing Input Clustering (MIMIC)}

Mutual-information-maximizing input clustering (MIMIC; \cite{BonetIV96MIMIC}) is a bivariate estimation-of-distribution algorithm (EDA). The Bayesian network of the probabilistic model of MIMIC can be represented as a directed path over~$n$ nodes, where each of the nodes corresponds to one of the~$n$ bit positions of~$f$. Further, MIMIC has two parameters, $\lambda, \mu \in \N$ with $\lambda \geq \mu$, that represent how many individuals are generated and selected each iteration, respectively.
%The first node in the path contains the probability to sample a~$1$ at the corresponding position. Each following node contains two probabilities: one to sample a~$1$ conditional on the predecessor in the path sampling a~$0$, and another one to sample a~$1$ conditional on the predecessor sampling a~$1$. A bit string is then generated from this model by following the path and sampling bits according to the probabilities at the nodes.

Initially, the model represents the uniform distribution. It is rebuilt each iteration in the following way: first, $\lambda$ individuals are generated according to the current model, and $\mu$ individuals are selected according to some selection mechanism. We call the resulting (multi-)set~$S$. A path is constructed greedily based on the entropy of the distribution of the bits at the different positions in~$S$.
%Informally speaking, the entropy denotes how close a distribution is to a deterministic value. A low entropy means that the distribution displays almost only a single value, whereas a high entropy means that many different values are observed.

The first node of the new path is a position with the lowest entropy, that is, a position with the largest number of~$1$s or~$0$s. Each subsequent node is chosen with respect to the lowest entropy conditional on the distribution of the current last node in the path. This way, the new path represents a model that best reflects the distributions of pairs of positions observed in~$S$. We now go into detail about our implementation of MIMIC (\Cref{alg:mimic}).

\begin{algorithm}[t]
    \caption{\label{alg:mimic} MIMIC~\cite{BonetIV96MIMIC} with parameters~$\mu$ and~$\lambda$, $\mu \leq \lambda$, and a selection scheme $\mathrm{select}_\mu$, optimizing a fitness function $f\colon \{0, 1\}^n \to \R$ with $n \geq 2$.}

    $t \gets 0$\;
    $\pi^{(t)} \gets \mathrm{id}_n$\;
    $P^{(t)} \gets (\frac{1}{2})_{i \in [n], b \in \{0, 1\}}$\;

    \Repeat%( $\vartriangleright$ \emph{iteration}~$t$)
    {\emph{termination criterion met}}
    {
        $O^{(t)} \gets \emptyset$\;
        \For{$i \in [\lambda]$}
        {
            $x^{(i)} \sim \mathrm{sample}_{\pi^{(t)}} \big(P^{(t)}\big)$\;
            $O^{(t)} \gets O^{(t)} \cup \{x^{(i)}\}$\;
        }
        $S^{(t)} \gets \mathrm{select}_\mu(O^{(t)}, f)$\;

        $I \gets [n]$\;
        $\pi^{(t + 1)}(1) \gets \arg\min_{i \in I} h[S^{(t)}; i]$\;
        $I \gets I \setminus \{\pi^{(t + 1)}(1)\}$\;

        \lFor{$b \in \{0, 1\}$}
        {
            $P^{(t + 1)}_{\pi^{(t + 1)}(1), b} \gets \gamma_1[S^{(t)}; \pi^{(t + 1)}(1)]$%
        }
        \For{$j \in [2 .. n]$}
        {
            $\pi^{(t + 1)}(j) \gets \arg\min_{i \in I} h[S^{(t)}; i \mid \pi^{(t + 1)}(j - 1)]$\;
            $I \gets I \setminus \{\pi^{(t + 1)}(j)\}$\;
            \lFor{$b \in \{0, 1\}$}
            {
                $P^{(t + 1)}_{\pi^{(t + 1)}(j), b} \gets \gamma_{1b}[S^{(t)}; \pi^{(t + 1)}(j) \mid \pi^{(t + 1)}(j - 1)]$%
            }
        }
        restrict all values of~$P^{(t + 1)}$ to the interval $[\tfrac{1}{n}, 1 - \tfrac{1}{n}]$\;
        $t \gets t + 1$\;
    }
\end{algorithm}

\subsubsection{Probabilistic model and sampling}
\label{sec:MIMICimplementation}

For our implementation of MIMIC, we describe the probabilistic model via a permutation~$\pi$ (over~$[n]$) and an $n \times 2$ matrix of probabilities. Bit strings are sampled bit by bit in the order of~$\pi$. For a position $i \in [2 .. n]$ and a bit value $b \in \{0, 1\}$, an entry $P_{\pi(i), b}$ denotes the probability to sample a~$1$ at position~$\pi(i)$, given that the bit at position~$\pi(i - 1)$ is~$b$. Note that entries in~$P$ always denote the probability to sample a~$1$. For the position~$\pi(1)$ (which does not have a predecessor in~$\pi$), we set $P_{\pi(1), 0} = P_{\pi(1), 1}$. Thus, either entry denotes the probability to sample a~$1$ without a prior.

For a bit string $x \in \{0, 1\}^n$, we write $x \sim \mathrm{sample}_{\pi}(P)$ to denote that~$x$ is being sampled with respect to the probabilistic model consisting of~$\pi$ and~$P$. More formally the sampling procedure creates~$x$ such that, for any bit string $y \in \{0, 1\}^n$,
\begin{align*}
    \Pr[x = y] &= (P_{\pi(1), 0})^{y_{\pi(1)}} \cdot (1 - P_{\pi(1), 0})^{1 - y_{\pi(1)}}\\
    &\quad \cdot \prod_{\substack{i \in [2 .. n]\colon\\y_{\pi(i)} = 0}} (1 - P_{\pi(i), y_{\pi(i - 1)}}) \cdot \hspace*{-1 em}\prod_{\substack{i \in [2 .. n]\colon\\y_{\pi(i)} = 1}} P_{\pi(i), y_{\pi(i - 1)}}.
\end{align*}

\subsubsection{Selection}

Given a population $O \subseteq \{0, 1\}^n$ of individuals and a fitness function~$f$, we write $\mathrm{select}_\mu(O, f)$ to denote a selection mechanism that selects~$\mu$ individuals from~$O$. In this paper, we use \emph{truncation selection,} that is, we sort the individuals in~$O$ by fitness and then select the~$\mu$ best individuals (breaking ties uniformly at random).

\subsubsection{Building the probabilistic model}
\label{sec:modelBuilding}

When constructing a new probabilistic model, MIMIC makes use of the unconditional and conditional (empirical) entropy of a set of bit strings. These mathematical functions make use of the relative occurrences of bit values. To this end, for a population $S \subseteq \{0, 1\}^n$, a position $i \in [n]$, and a bit value $b \in \{0, 1\}$, let the frequency of~$b$ at position~$i$ in~$S$ be
\[
    \gamma_b[S; i] = \frac{1}{|S|} \sum_{x \in S} \mathds{1}\{x_i = b\}.
\]
Further, for a population $S \subseteq \{0, 1\}^n$, two positions $i, j \in [n]$, and two bit values $b_1, b_2 \in \{0, 1\}$, we define the conditional frequency of~$b_1$ at position~$i$ in~$O$ conditional on the value~$b_2$ at position~$j$ by
\begin{align*}
    &\gamma_{b_1 b_2}[S; i \mid j] =\\
    &\hspace*{3 em}
    \begin{cases}
        \frac{1}{2} & \hspace*{-3 em}\textrm{if } \gamma_{b_2}[S; j] = 0,\\
        \frac{1}{|S| \cdot \gamma_{b_2}[S; j]} \sum_{x \in S} \mathds{1}\{x_i = b_1 \land x_j = b_2\}  & \hfill\textrm{else.}
    \end{cases}
\end{align*}
Note that the case $\gamma_{b_2}[S; j] = 0$ means that the event we condition on has a probability of~$0$, which is not well defined. In order to represent our lack of knowledge in this case, we choose~$\frac{1}{2}$ as the value for the respective probability, which corresponds to a uniform distribution.

We now define the (empirical) entropy functions that MIMIC utilizes. To this end, we define that $0 \cdot \log_2(0) = 0$. For a population $S \subseteq \{0, 1\}^n$ and a position $i \in [n]$, the entropy at position~$i$ in~$S$ is
\[
    h[S; i] = - \sum_{\mathclap{b \in \{0, 1\}}} \gamma_b[S; i] \cdot \log_2(\gamma_b[S; i]).
\]
Further, for a population $S \subseteq \{0, 1\}^n$ and two positions $i, j \in [n]$, the entropy at position~$i$ in~$O$ conditional on position~$j$ is
\[
    h[S; i \mid j] = -\sum_{\mathclap{(b_1, b_2) \in \{0, 1\}^2}} \gamma_{b_1 b_2}[S; i \mid j] \cdot \gamma_{b_2}[S; j] \cdot \log_2(\gamma_{b_1 b_2}[S; i \mid j]).
\]

Given these definitions and a population $S \subseteq \{0, 1\}^n$ of selected individuals, MIMIC builds a new model by constructing a new permutation~$\pi'$ and updating the probabilities in~$P$ with respect to~$\pi'$. The permutation~$\pi'$ is built in the following iterative and greedy fashion, breaking ties uniformly at random: for the first position, an index with the lowest entropy in~$S$ is chosen. Each subsequent position is determined by an index with the lowest entropy in~$S$ conditional on the previous index in~$\pi'$.

Each time that a new position~$i$ is determined for~$\pi'$, the probabilities~$P_{i, 0}$ and~$P_{i, 1}$ are updated. If $i = \pi'(1)$, both~$P_{i, 0}$ and~$P_{i, 1}$ are set to the relative number of~$1$s at position~$i$ in~$S$, that is $\gamma_1[S; i]$. If $i \neq \pi'(1)$, that is, there is a preceding position~$j$ in~$\pi'$, for a bit value~$b$, the probability~$P_{i, b}$ is set to the relative number of~$1$s at position~$i$ in~$S$ that also have a value of~$b$ at position~$j$. Note that this is equivalent to setting $P_{i, b}$ to $\gamma_{1b}[S; i \mid j]$.

In order to circumvent the model from sampling only~$0$s or only~$1$s at some position, we make sure that no probability is~$0$ or~$1$. We enforce this after building~$\pi'$ and updating~$P$ by increasing probabilities less than~$\frac{1}{n}$ to~$\frac{1}{n}$ and by decreasing probabilities greater than $1 - \frac{1}{n}$ to $1 - \frac{1}{n}$. We may also say that we restrict~$P$ to the interval $[\frac{1}{n}, 1 - \frac{1}{n}]$.

Note that restricting the probabilities makes it necessary to define a value for the first case in the definition of $\gamma_{b_1 b_2}$, since it can happen that $\gamma_{b_2}[S; j] = 0$ in~$S$, but the corresponding probability is not~$0$, as it is restricted to $[\frac{1}{n}, 1 - \frac{1}{n}]$. In such a case, it is possible to sample~$b_2$ with the new model, making it necessary to define the probability $P_{i, b_2}$.

\subsection{\UBM (\ubm)}
\label{sec:UBM}

Many benchmark functions test an algorithm's capability of finding an optimal solution at all. Hence, they are commonly composed of deceptive or otherwise hard landscapes with many dependencies. In order to reduce the probability of finding an optimal solution by pure chance, the number of optima of such a function is usually small. For EDAs, it is not only interesting how fast they find good solutions but also how well their probabilistic model represents the distribution of good solutions in the search space.

To this end, we introduce the test function \ubm. It represents a fairly simple hill-climbing landscape, similar to that of the well-known \OM function (the sum of all bit values in an individual), but features an exponential number of optima. Thus, finding a single optimal solution is easy, but exploiting the structure of \ubm and being able to generate a large number of \emph{different} optima is challenging.
%For EDAs, this task is especially interesting, as one can investigate how well their probabilistic model is capable to implicitly represent the set of \emph{all} optima.

\subsubsection{Definition}
Given a bit string of length~$n$, \ubm operates on blocks of size~$2$ and returns the number of blocks that are either~$00$ or~$11$. Let~$n$ be even. For each $j \in [\frac{n}{2}]$, let the pair of positions $2j - 1$ and $2j$ denote block~$j$. For an individual $x \in \{0, 1\}^n$, we say that block~$j$ is \emph{correct} if the bits in block~$j$ have identical values. The objective of \ubm is to maximize the number of correct blocks. Formally, for all $x \in \{0, 1\}^n$,
\[
    \ubm(x) = \sum_{j \in [n/2]} \mathds{1}\{x_{2j - 1} = x_{2j}\}.
\]

Consequently, \ubm has a maximal fitness of~$\frac{n}{2}$ and $2^{n/2}$ different optima, since there are two possibilities for each of the $\frac{n}{2}$ blocks to be correct.

\subsection{An Ideal Model of MIMIC for \ubm}
\label{sec:idealModel}
We are interested in a model of MIMIC that generates each optimal solution of \ubm with the same maximal probability. We call such a model \emph{ideal.}
%In the following, we go into detail about how such a model looks like and what desirable properties it features. We use these properties in \Cref{sec:experiments}, where we present our results.

The permutation~$\pi$ of an ideal model is such that, for each block $j \in [\frac{n}{2}]$ of \ubm, the positions $2j - 1$ and $2j$ are adjacent in~$\pi$ (but in any order). In the following, assume without loss of generality that $\pi(2j - 1) < \pi(2j)$, that is, position $2j - 1$ occurs before $2j$ in~$\pi$.  For the probability matrix~$P$ of an ideal model, the probabilities of position $2j - 1$ are both~$\frac{1}{2}$, and the probabilities of position~$2j$ are $1 - \frac{1}{n}$ (conditional on a prior~$1$) and~$\frac{1}{n}$ (conditional on a prior~$0$). Note that, when sampling a solution with an ideal model, the bit sampled at position~$2j$ is sampled conditional on the bit at position $2j - 1$. Due to the choice of~$P$, this probability is maximized. Choosing~$\frac{1}{2}$ as the value of \emph{both} probabilities of position $2j - 1$ further ensures two things: (1) The bit at position $2j - 1$ is sampled independently of the bit at position $2j - 2$.\footnote{For this to hold, it suffices that both probabilities of position $2j - 1$ are the same; they do not have to be~$\frac{1}{2}$.} (2) Block~$j$ is $00$ or $11$ with equal probability. Overall, an ideal model has maximal equal probability to sample an optimum. We now discuss features that help in assessing whether a model is close to an ideal model or not.

In an ideal model, the probability that a generated bit string is one of the $2^{n/2}$ optima is $2^{n/2} \big(\frac{1}{2} (1 - \frac{1}{n})\big)^{n/2} = \big((1 - \frac{1}{n})^{n}\big)^{1/2}$\!. Using that $\lim_{n \to \infty} (1 - \frac{1}{n})^n = \frac{1}{\mathrm{e}}$, the probability of MIMIC to sample an optimal solution, given an optimal model, is roughly $1 / \sqrt{e} \approx 60.65\,\%$. However, note that the probability of~$1/\sqrt{\mathrm{e}}$ of sampling an optimum is, by itself, \emph{not} indicative of an ideal model. This probability is also achieved by any other model which is like an ideal model but has the following difference: for each block~$j$ (defined as above), the probabilities at position $2j - 1$ are equal but not necessarily~$\frac{1}{2}$. Given such a model, the probability to sample \emph{any} optimum is still~$1/\sqrt{\mathrm{e}}$. However, the probability to sample a \emph{specific} optimum may differ from optimum to optimum. Consequently, we also consider a second indicator for an ideal model.

The property of an ideal model that \emph{each} optimum has the same probability of being sampled makes it unlikely that such a model creates duplicate solutions in $m \in \N^+$ independent tries. More formally, for an optimal model, since each optimum is equally likely, the probability that all optima are distinct when sampling~$m$ optimal solutions is $(2^{n/2})!/\big(2^{m n/2} \cdot (2^{n/2} - m)!\big)$, by the birthday paradox. This probability is at least $(1 - m/2^{n/2})^m \geq 1 - m^2/2^{n/2}$, by Bernoulli's inequality, which is close to~$1$ as long as $m^2 = o(2^{n/2})$.

We conclude from these insights that a good model of MIMIC should sample optima with a probability of roughly $1/\sqrt{\mathrm{e}}$ and that it should not sample duplicates, with high probability.

\section{Results}
\label{sec:experiments}

In this section, we show that MIMIC creates models in reasonable time that behave similarly to an ideal model for \ubm. We first explain our setup, then we discuss our results.

\subsection{Algorithm Setup}

We use MIMIC as seen in \Cref{alg:mimic} with truncation selection (with uniform tie-breaking) and with $\lambda = \lfloor 12 n \ln n \rfloor$ and $\mu = \lfloor \lambda / 8 \rfloor$.
%Our reasoning for the choices of~$\mu$ and~$\lambda$ are based on the assumptions that MIMIC with truncation selection is similar to the \emph{univariate marginal distribution algorithm} (UMDA; \cite{MuehlenbeinP96UMDA}) and that \ubm is similar to \OM.
%
Our choice for~$\lambda$ is based on a grid search for the $n$-factor in the interval $[1, 20]$ with a step size of~$1$. The value~$12$ was the first with which MIMIC found an optimum in all runs of our test setup (see also \Cref{sec:setup}). For~$\mu$, we chose a constant fraction of~$\lambda$, which is common for EDAs.

%For $\mu$, we assume that, per block $j \in [\frac{n}{2}]$ of \ubm, the probabilities of one position, without loss of generality $2j$, stay close to~$\frac{1}{2}$, since the blocks are independent in \ubm and since~$0$s and~$1$s are treated equally. Assume that the probabilities of position $2j$ stay in $[\frac{1}{4}, \frac{3}{4}]$, whereas the probabilities of position $2j - 1$ go to $1 - \frac{1}{n}$ and~$\frac{1}{n}$, after an initial random walk in the interval $[\frac{1}{4}, \frac{3}{4}]$ when MIMIC did not detect that both positions are strongly correlated. The probability that a generated individual has a~$00$ or a~$11$ in block~$j$ is then at least $2(\frac{1}{4})^2 = \frac{1}{8}$. Thus, choosing $\mu = \lambda / 8$ should guarantee that many selected individuals have the same bit values in block~$j$.

\subsection{Test Setup}
\label{sec:setup}

We are interested in determining how well MIMIC is capable of generating a probabilistic model that implicitly captures an exponential number of optima of \ubm. Consequently, we use our insights from \Cref{sec:idealModel} in order to determine how good a model of MIMIC is. To this end, we let MIMIC run for a number of iterations~$I$, which we explain below, and we determine
\begin{enumerate}
    \item the probabilistic model (that is~$\pi$ and~$P$) in each iteration,
    \item with what probability optimal solutions are created in each iteration, and
    \item how many distinct optima are created in~$I$ iterations.
\end{enumerate}

Our choice of~$I$ is as follows: let~$T$ denote the number of iterations until MIMIC samples an optimum for the first time. Then we let the algorithm run for~$T$ more iterations, that is, $I = 2T$. Since MIMIC might fail finding an optimum in a reasonable time, we abort a run if the number of iterations exceeds $50\,000$ iterations. However, we chose~$\lambda$ and~$\mu$ such that \emph{all} of our tests were successful. That is, MIMIC always found an optimum, and we let the algorithm run for~$2T$ iterations.

We consider MIMIC for values of~$n$ from~$50$ to~$200$ in steps of~$10$. For each value of~$n$, we start~$100$ independent runs. For each run, we record the number of iterations until the first optimum is sampled (that is,~$T$), the set of \emph{all} optima that are found in each of the $2T$ total iterations (which may include duplicates), the number of optima found in each of these iterations, as well as the probabilistic model in each iteration. Note that with this data we are able to compute the information above we are interested in.

\subsubsection{Visualization}
\label{sec:visualziation}
We depict our results in \Cref{fig:runTime,fig:numberOfOptima,fig:probabilityOfOptimum,fig:extremeValues} and in \Cref{tab:firstBlockProbabilities,tab:model}. In these plots, we visualize:
\begin{enumerate}
    \item the total number of iterations and fitness evaluations,
    \item how the probabilistic model evolves during a run,
    \item the number of optima found as well the number of runs that only found distinct optima, and
    \item the probability of sampling an optimum during an iteration.
\end{enumerate}

For each figure, we plot the data of all $100$~runs (per~$n$) simultaneously in a concise manner: we depict the median of the data as a point and connect the medians with a solid line. Further, we depict the mid $50\,\%$ (that is, ranks~$25$ to~$75$ when ordering the runs) as a shaded area bounded by a dotted line. We provide more information about the visualization in the discussion of our results.

\subsection{Discussion}
\label{sec:dicussion}
In this section, we discuss the results depicted in \Cref{fig:runTime,fig:numberOfOptima,fig:probabilityOfOptimum,fig:extremeValues} and in \Cref{tab:firstBlockProbabilities,tab:model}.

\subsubsection{Run time}

\begin{figure*}
    \begin{subfigure}[t]{0.45\textwidth}
        \includegraphics[width = \textwidth]{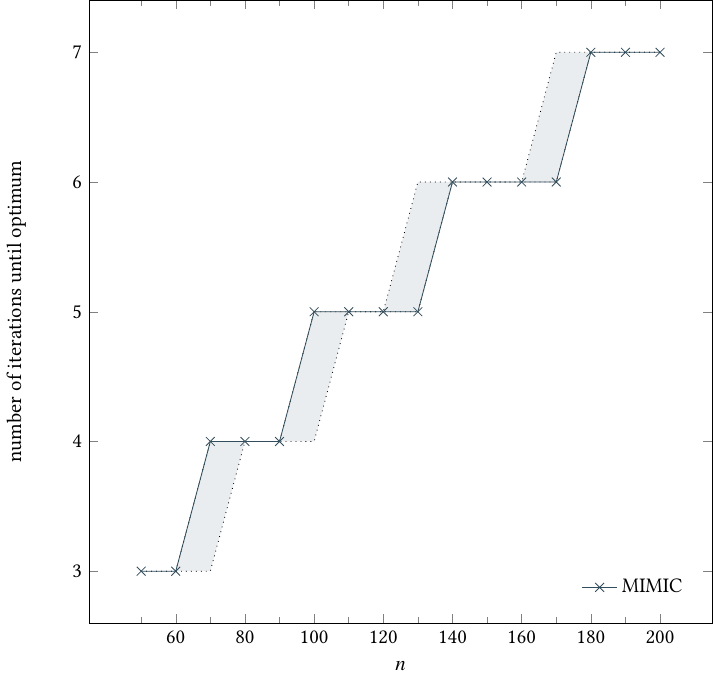}
        \caption{\label{fig:runTimeIterations}The number of iterations it took each of the $100$ runs per~$n$ until an optimum was found for the first time (that is, $T$).}
    \end{subfigure}
    \hfil
    \begin{subfigure}[t]{0.45\textwidth}
        \includegraphics[width = \textwidth]{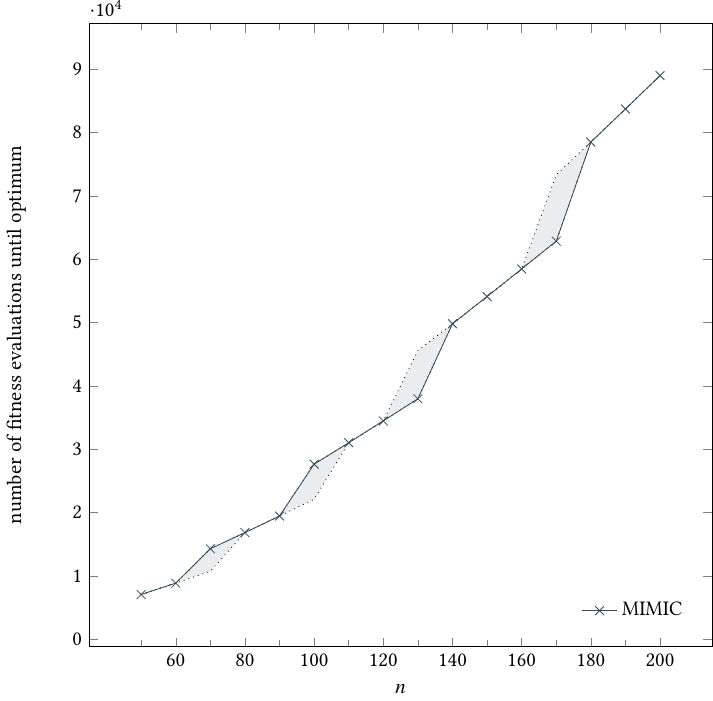}
        \caption{\label{fig:runTimeFitness}The number of fitness evaluations it took each of the $100$ runs per~$n$ until an optimum was found for the first time (that is, $\lambda T$).}
    \end{subfigure}
    \caption{\label{fig:runTime}Two depictions of the run time of MIMIC optimizing \ubm. For information about the type of plot used, please refer to \Cref{sec:visualziation}. For a discussion of these plots, please refer to \Cref{sec:runTime}.}
\end{figure*}

\label{sec:runTime}
\Cref{fig:runTime} shows the run time of each of the $100$ runs per~$n$ with respect to the number of iterations (\Cref{fig:runTimeIterations}) and with respect to the number of fitness function evaluations (\Cref{fig:runTimeFitness}).

The number of iterations depicted is the number of iterations of each run until an optimum was found for the first time. That is, the number of iterations corresponds to~$T$, as explained in \Cref{sec:setup}. For each change in the number of iterations (for example, at $n = 70$), there is one value of~$n$ that has a high variance (the shaded area), and many runs take either the number of iterations of the previous value of~$n$ or an extra iteration. Except for these transitions, the run time of MIMIC is enormously consistent, with the mid $50\,\%$ all taking the same number of iterations. Overall, the number of iterations slightly increases with~$n$.

The number of fitness function evaluations provides a better picture on how long MIMIC takes for a run. Note that the numbers shown in \Cref{fig:runTimeFitness} are the numbers from \Cref{fig:runTimeIterations} times~$\lambda$, as MIMIC performs~$\lambda$ fitness evaluations in each iteration. The reason that the curve is not constant when the number of iterations stays the same for different values of~$n$ is that we chose $\lambda = \lfloor 12 n \ln n \rfloor$, which grows in~$n$. Thus, depending on how~$T$ grows in~$n$, the total run time of MIMIC on \ubm is at least in the order of $n \ln n$.

\subsubsection{Probabilistic model}
\label{sec:probabilisticModel}

%\begin{figure*}
%    \begin{subfigure}[t]{0.45\textwidth}
%        \includegraphics[width = \textwidth]{../newExperiments/plots/UBM/probabilisticModel/firstPosition/n=110probabilityPositions.pdf}
%        \caption{\label{fig:probabilityValues110}The four different probabilities of the positions~$1$ and~$2$ for $n = 110$.}
%    \end{subfigure}
%    \hfil
%    \begin{subfigure}[t]{0.45\textwidth}
%        \includegraphics[width = \textwidth]{../newExperiments/plots/UBM/probabilisticModel/firstPosition/n=200probabilityPositions.pdf}
%        \caption{\label{fig:probabilityValues200}The four different probabilities of the positions~$1$ and~$2$ for $n = 200$.}
%    \end{subfigure}
%    \caption{\label{fig:probabilityValues}The probabilities of the first block in \ubm in each iteration, for $n = 110$ and $n = 200$. In each of the two plots, the two horizontal lines (at the top and the bottom) depict the probability borders~$\frac{1}{n}$ and $1 - \frac{1}{n}$; the dashed vertical line represents the first iteration that an optimum was sampled (see \Cref{fig:runTimeIterations}). For a discussion of these plots, please refer to \Cref{sec:probabilisticModel}.}
%\end{figure*}

\begin{figure}
    \includegraphics[width = \columnwidth]{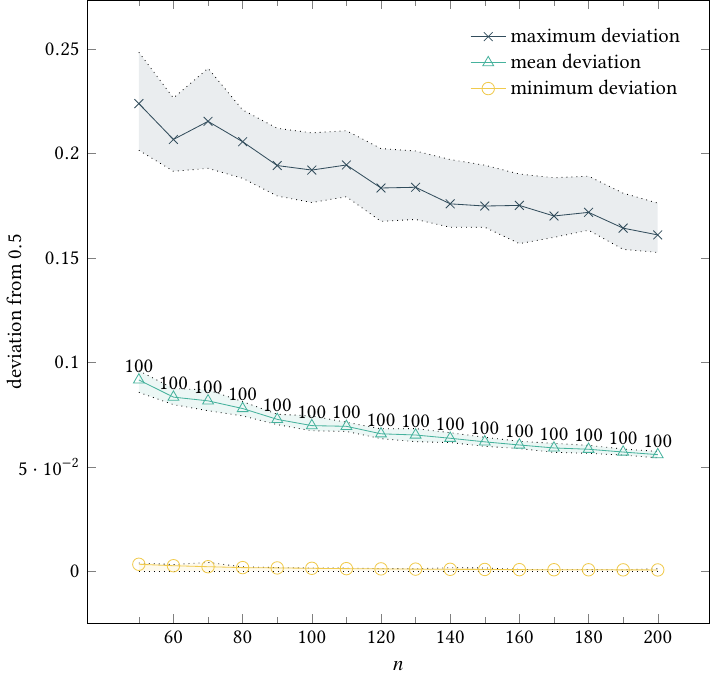}
    \caption{\label{fig:extremeValues}Depicted are the maximum, mean, and minimum of the deviation of the central probabilities in~$P$ from~$0.5$ in iteration~$2T$. The numbers over the plot with the triangles denote the number of runs (out of~$100$) that have a correct permutation in their model. For a discussion of this plot, please refer to \Cref{sec:probabilisticModel}.}
\end{figure}

\begin{figure*}
    \begin{subfigure}[t]{0.45\textwidth}
        \includegraphics[width = \textwidth]{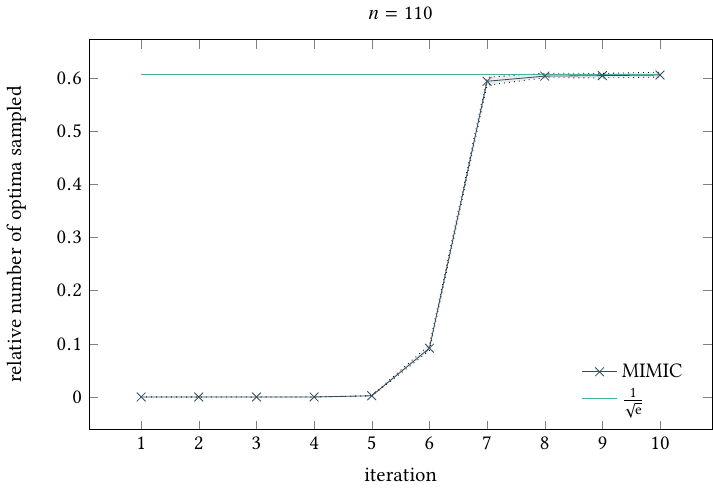}
        \caption{\label{fig:probabilityOfOptimum110}The relative number of optima in an iteration for $n = 110$.}
    \end{subfigure}
    \hfil
    \begin{subfigure}[t]{0.45\textwidth}
        \includegraphics[width = \textwidth]{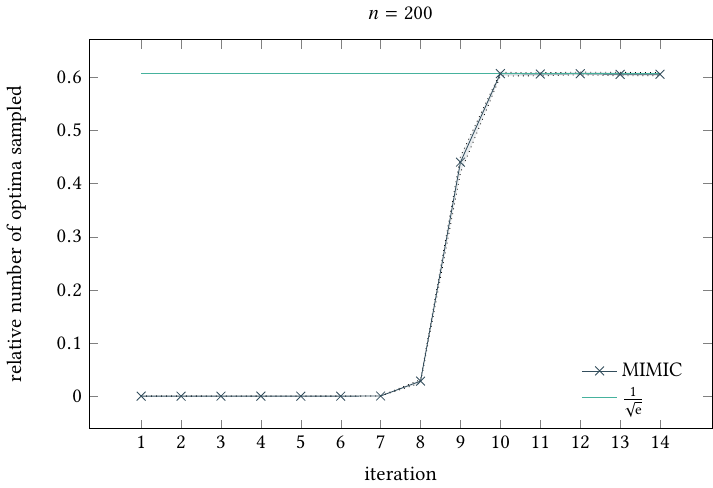}
        \caption{\label{fig:probabilityOfOptimum200}The relative number of optima in an iteration for $n = 200$.}
    \end{subfigure}
    \caption{\label{fig:probabilityOfOptimum}Depicted are how the relative number of optima sampled evolves over the number of iterations for MIMIC optimizing \ubm. The horizontal line at the top shows the value $1/\sqrt{\mathrm{e}} \approx 60.65\,\%$, which is roughly the probability of sampling an optimum in a single iteration, given an ideal model of MIMIC for \ubm (see \Cref{sec:idealModel}). For a discussion of these plots, please refer to \Cref{sec:numberOfOptima}.}
\end{figure*}

\begin{figure}
    \includegraphics[width = \columnwidth]{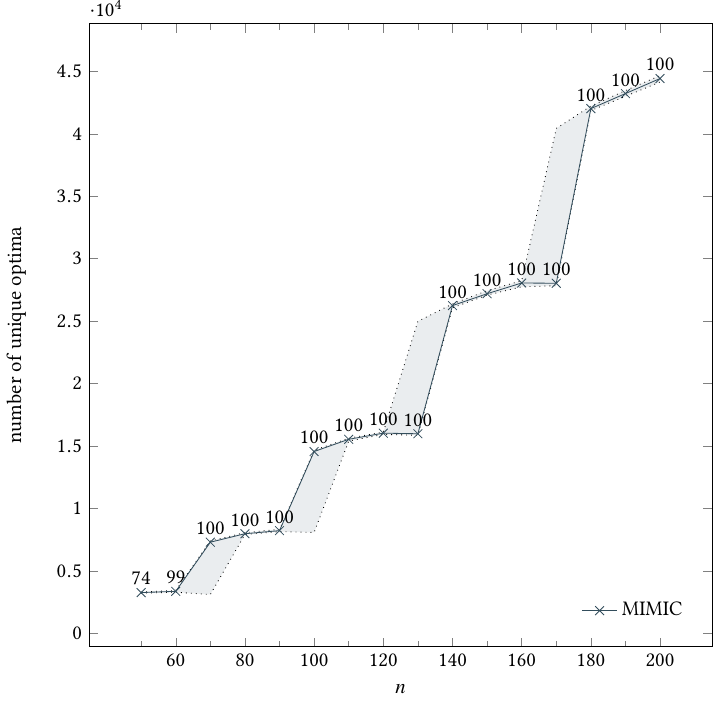}
    \caption{\label{fig:numberOfOptima}The number of distinct optima that MIMIC found when optimizing \ubm. An optimum is distinct if it was only sampled once during a single run. The number over each data point states how many of the $100$ runs sampled \emph{exclusively} distinct optima. For a discussion of these plots, please refer to \Cref{sec:numberOfOptima}.}
\end{figure}

\begin{table}
    \centering
    \caption{\label{tab:model}The probabilities of the first $10$ positions (occurring in~$\pi$) of MIMIC optimizing \ubm for one of the runs with $n = 200$, at iteration~$2T$. For a discussion of this table, please refer to \Cref{sec:probabilisticModel}.}
    \begin{tabular}{lrr}
        position~$i$ & $P_{i, 0}$ & $P_{i, 1}$\\ \midrule
        17 & 0.662052 & 0.662052\\
        18 & 0.005 & 0.995\\
        90 & 0.636872 & 0.610266\\
        89 & 0.005 & 0.995\\
        41 & 0.649587 & 0.58435\\
        42 & 0.005 & 0.995\\
        49 & 0.68438 & 0.546488\\
        50 & 0.005 & 0.995\\
        104 & 0.582677 & 0.613208\\
        103 & 0.005 & 0.995\\
        & $\vdots$ &
    \end{tabular}
\end{table}

\begin{table}
    \centering
    \caption{\label{tab:firstBlockProbabilities}The probabilities of the positions~$1$ and~$2$ (occurring in~$\pi$) of MIMIC optimizing \ubm for one of the runs with $n = 200$, over all iterations. For a discussion of this table, please refer to \Cref{sec:probabilisticModel}.}
        \begin{tabular}{lrrrr}
            iteration & $P_{1, 0}$ & $P_{1, 1}$ & $P_{2, 0}$ & $P_{2, 1}$\\ \midrule
            1 & 0.5 & 0.5 & 0.5 & 0.5 \\
            2 & 0.533825 & 0.454897 & 0.428928 & 0.587039 \\
            3 & 0.311688 & 0.665446 & 0.542926 & 0.485564 \\
            4 & 0.51094 & 0.458128 & 0.256098 & 0.789337 \\
            5 & 0.141582 & 0.828571 & 0.535533 & 0.478152 \\
            6 & 0.474968 & 0.453086 & 0.115023 & 0.903664 \\
            7 & 0.465116 & 0.519018 & 0.0794045 & 0.943806 \\
            8 & 0.465385 & 0.478368 & 0.0381406 & 0.98 \\
            9 & 0.450299 & 0.486737 & 0.00945626 & 0.995 \\
            10 & 0.440618 & 0.480589 & 0.005 & 0.995 \\
            11 & 0.005 & 0.995 & 0.442663 & 0.462725 \\
            12 & 0.005 & 0.995 & 0.470277 & 0.424279 \\
            13 & 0.406593 & 0.460972 & 0.005 & 0.995 \\
            14 & 0.455733 & 0.421111 & 0.005 & 0.995
        \end{tabular}
\end{table}

\Cref{fig:extremeValues} and \Cref{tab:firstBlockProbabilities,tab:model} showcase information about the probabilistic model of MIMIC and its quality with respect to an ideal model (see also \Cref{sec:idealModel}). For a comparison to make sense, it is important that the permutation~$\pi$ of a model of MIMIC is close to that of an ideal model~-- ideally, $\pi$ would correspond to a permutation of an ideal model. To this end, we say that a permutation~$\pi$ is \emph{correct} if, starting from the first position, its positions occur in pairs of two such that (1) the positions in each pair differ by exactly~$1$ and that (2) the maximum of the positions of each pair is an even number. Note that the set of all correct permutations corresponds exactly to that of all ideal models.\footnote{Property (2) is necessary, since \ubm defines its block with respect to position~$1$. For example, positions~$1$ and~$2$ form a block in \ubm, but positions~$2$ and~$3$ do not.}

In \Cref{tab:model}, we show an excerpt of the model from one out of the $100$~runs of MIMIC on \ubm in the last iteration, that is, $2T$. In total, we mention $10$ entries from the model (out of $200$). The first column depicts the bit positions as they occur in the permutation~$\pi$. We see that all entries occur as they would in a correct permutation, suggesting that the entire permutation is correct. Note that the order of the positions per pair appears randomly, which makes sense, as the order does not matter for sampling a block in \ubm correctly.

The other two columns of \Cref{tab:model} show the two probabilities of the position from the first column. We see that, for each pair of positions (as defined above), the first position has its probabilities close to~$0.5$ and second one has its probabilities at the borders of the interval $[\frac{1}{n}, 1 - \frac{1}{n}]$. Further, the probabilities at the borders are at the correct end for maximizing the probability of sampling a block in \ubm correctly. That is, the probability $P_{i, 0}$ is at~$\frac{1}{n}$ (making it likely to sample a~$0$ when the previous position sampled at~$0$), and the probability $P_{i, 1}$ is at $1 - \frac{1}{n}$. Overall, the results from \Cref{tab:model} already suggest that MIMIC builds a model close to an ideal one.

In \Cref{fig:extremeValues}, we have a closer look at how closely the model of MIMIC in iteration~$2T$ resembles an ideal model. In order for such a comparison to make sense, we first analyze how well the permutation of such a model deviates from the permutation of an ideal model. Out of \emph{all} of our runs, \emph{each} run produced a correct permutation in iteration~$2T$. We depict these numbers in \Cref{fig:extremeValues} over the curve in the middle, with the triangles. Thus, the only way for a model of MIMIC to deviate from an ideal model is in how largely the probabilities in~$P$ deviate from those of an ideal model.

When comparing probabilities of~$P$ to that of an ideal model, we group the probabilities into those that should be close to~$0.5$ (the \emph{central} probabilities) and into those that should be close to the borders (the \emph{border} probabilities). We may also use the respective adjective for a position in order to indicate that both of the probabilities are central or border. We group the probabilities with respect to the blocks in~$\pi$. In order to determine which position of each block is central and which is border, we look at the probability with the highest deviation from~$0.5$ (breaking ties uniformly at random). The position with the probability that has the highest deviation is considered border, the other position is considered central.

We then calculate the absolute distance of each probability to its ideal value. For the central probabilities, we calculate their distance to~$0.5$ (regardless of whether the probability is conditional on a~$0$ or a~$1$). For the border probabilities conditional on a~$0$, we calculate their distance to~$\frac{1}{n}$, and for those conditional on a~$1$, we calculate their distance to $1 - \frac{1}{n}$. Afterward, for the two groups of central and border probabilities, we calculate, for each of the positions per run and value of~$n$, the maximum, mean, and minimum of the deviations of each probability.

The results of these calculations for the central probabilities are depicted in \Cref{fig:extremeValues}. The arguably most interesting result is the maximum deviation among all positions of a single run. This value seems to decrease with increasing~$n$. However, a deviation of about~$0.2$ can still be considered rather large. We discuss in the following sections how this affects the quality of the model.

The deviations for the border probabilities are not depicted, as the maximum over all runs and all values of~$n$ was in the order of $10^{-6}$. This suggests that the border probabilities are always very close to the borders in iteration~$2T$.

Since we only looked at the model of MIMIC in iteration~$2T$, \Cref{tab:firstBlockProbabilities} provides an excerpt of how the probabilities of the first block evolve over the iterations. We depict data from one of the runs with $n = 200$. From iteration~$10$ to~$11$ and from~$12$ to~$13$, we see that the probabilities of the positions~$1$ and~$2$ change their statuses of being central or border. This makes sense, as we already briefly discussed, as the order of the positions in a block does not matter for sampling a correct block. Given a correct block, it is then random which position MIMIC determines to be the first in its permutation (and, thus, central) and which it chooses next (being border). Thus, we conclude that MIMIC does not converge to a single model that is close to an ideal model but instead switches between different models from iteration to iteration.

\subsubsection{Similarity to an ideal model}
\label{sec:numberOfOptima}

Since the results so far suggest that the model of MIMIC is close to an ideal model except for the deviation of the central probabilities (see \Cref{sec:probabilisticModel}), we now consider how well the model reflects the two properties of an ideal model that we describe in \Cref{sec:idealModel}. We start with the probability to sample an optimum in each iteration.

\Cref{fig:probabilityOfOptimum} shows how many of the solutions of the~$\lambda$ solutions sampled during each iteration are optima. We chose to depict this ratio for the cases of $n = 110$ and $n = 200$, which are cases where all of the mid $50\,\%$ of the runs used the same number of iterations (see also \Cref{fig:runTimeIterations}). This data can be interpreted as the probability of sampling an optimum in each iteration. Following our ideas discussed in \Cref{sec:idealModel}, we also depict the value $1/\sqrt{\mathrm{e}}$ in both plots, which represents the probability to sample an optimum, given an ideal model.

Both \Cref{fig:probabilityOfOptimum110,fig:probabilityOfOptimum200} show that the empirical ratio is surprisingly close to the ideal value. This suggests that the model behaves similarly to an ideal model in terms of consistently sampling optima, despite the central probabilities sometimes deviating somewhat largely from~$0.5$ (see \Cref{fig:extremeValues}). The fact that some data points show a ratio that is slightly higher than the theoretical optimum is due to the variance in the randomness of the algorithm.

\Cref{fig:numberOfOptima} shows how many of the optima that MIMIC found per run were distinct as well as how many runs only found distinct optima.
Note that the sudden increases in the number of optima found relate to the number of iterations depicted in \Cref{fig:runTimeIterations}.
Except for the cases $n = 50$ and $n = 60$, all optima that MIMIC found per run were distinct. This result is remarkable and suggests that MIMIC builds a very general model that is capable of sampling a huge variety of different solutions.

We now argue that it is not unlikely for the cases $n = 50$ and $n = 60$ to have runs that failed to only find distinct optima. In \Cref{sec:idealModel} we derived a lower bound on how likely it is to have \emph{no} duplicate in~$m$ samples. In a similar fashion, one can derive an upper bound (using that, for $a \leq b$, $a!/(a - m)! \lesssim (a - \frac{m}{2})^m$ and that, for $x \in [0, 1]$, $(1 - x)^m \approx \mathrm{e}^{-xm}$) of roughly $\mathrm{e}^{-m^2/2^{n/2 + 1}}$. Thus, a lower bound of having a duplicate in~$m$ tries is roughly at least $1 - \mathrm{e}^{-m^2/2^{n/2 + 1}}$. For $n = 60$, using that~$4$ out of~$6$ iterations are used for sampling optima and that about $1.7 \cdot 10^4$ solutions are created in a run (retrieved from the data used for \Cref{fig:runTime}), we get that the probability for a run to have a duplicate optimum is about $6\,\%$, which means that we would expect about $6$ failures. For $n = 70$, the probability to have a duplicate optimum drops already below $1\,\%$.\footnote{This estimation makes the assumption that the model is ideal in~$4$ out of~$6$ iterations. However, data similar to that depicted in \Cref{fig:probabilityOfOptimum} suggests that it takes at least one iteration until the model samples optima consistently.}

Overall, the results from \Cref{fig:numberOfOptima,fig:probabilityOfOptimum} suggest that the model of MIMIC behaves similarly to an ideal model. We thus consider it to actually be similar to an ideal model.

\section{Theoretical Analyses}
\label{sec:theory}

We now show mathematically rigorously that all univariate EDAs perform poorly on \ubm. More specifically, we show that a univariate model (i)~has a very small probability of sampling an optimum of \ubm at all, or (ii)~with very high probability samples only in a ball of logarithmic radius. This forbids any performance close to what we showed for MIMIC.
%
%
 %in the sense that they only sample a few of the exponentially many optima.
%Our main result (\Cref{thm:number_of_optima_univariate}) states roughly that the diversity in the solutions created by a univariate model is indirectly proportional to the likelihood of creating an optimal solution for \ubm.
We note that this argument holds for \emph{all} univariate EDAs since it only refers to their probabilistic models, but not to the specific algorithm. The probabilistic model of a univariate EDA~$A$, when optimizing bit strings of length $n \in \N^+$, is fully characterized by a probability vector $p \in [0, 1]^n$, commonly called the \emph{frequency vector} of~$A$.
We denote the component of~$p$ at position $i \in [n]$ by~$p_i$, and we refer to it as a \emph{frequency}.
The EDA~$A$ creates a bit string $x \in \{0, 1\}^n$ via~$p$ by, for all $i \in [n]$, setting $x_i = 1$ with probability~$p_i$ and $x_i = 0$ with probability $1 - p_i$, independently of any other random choices.
Thus, for all $y \in \{0, 1\}^n$, it holds that
\begin{align*}
    \Pr[x = y] = \prod_{\substack{i \in [n]\colon\\ y_i = 0}} (1 - p_i) \cdot \prod_{\substack{i \in [n]\colon\\y_i = 1}} p_i\ .
\end{align*}
We write $x \sim \mathrm{sample}(p)$ to say that $x \in \{0,1\}^n$ is sampled according to the univariate model~$p$.

To prove our result, we make use of the a large-deviation bound due to Hoeffding~\cite{Hoeffding63}, here given in the version of~{\cite[Theorems~1.10.1 and 1.10.21]{Doerr20bookchapter}}.

\begin{theorem}
    \label{thm:chernoff_larger}
    Let $n \in \N$, $\delta \in \R_{> 0}$, and let~$X$ be the sum of~$n$ independent random variables, each taking values in $[0, 1]$. Let $\mu^+ \ge E[X]$. Then
    \begin{align*}
        \Pr\big[X \geq (1 + \delta) \mu^+ \big] \leq \exp\left(- \frac 13 \delta \mu^+ \right)\ .
    \end{align*}
\end{theorem}

For all $n \in \N^+$ and for all $x, y \in \{0, 1\}^n$, let $\hammingDistance{x}{y} = |\{i \in [n] \mid x_i \neq y_i\}|$, that is, the Hamming distance of~$x$ and~$y$. With these preparations, we can formally state the main result of this section.

\begin{theorem}
    \label{thm:number_of_optima_univariate}
    Let $n \in \N^+$ with~$n$ being even, and let~$p$ be the length-$n$ frequency vector of a univariate EDA~$A$. Let $x \sim \mathrm{sample}(p)$ and let~$E$ denote the event that~$x$ is an optimum of~\ubm.
    Assume that there is a $k \in \R^+$ such that $\Pr[E] \geq n^{-k}$.
    Last, let $z = (\lfloor 1/2 + p_i\rfloor)_{i \in [n]}$.
		Then for all $\gamma \ge 4k$, we have
		\[\Pr[\hammingDistance{x}{z} \ge \gamma \ln(n)] \le n^{-\gamma/6}.\]
\end{theorem}

\begin{proof}
    Let $X = \hammingDistance{x}{z}$, and note that~$X$ follows a Poisson binomial law with~$n$ trials where trial $i \in [n]$ has success probability $s_i \coloneqq \max\{p_i, 1 - p_i\}$.
    Thus, it follows that
    \begin{align}
        \label{eq:poisson_binomial_expectation}
        \E[X] = \sum\nolimits_{i \in [n]} s_i = \sum\nolimits_{i \in [n]} |p_i - z_i| = \|p - z\|_1\ .
    \end{align}

    We bound~$\E[X]$ from above and then apply \Cref{thm:chernoff_larger} to~$X$, bounding with high probability the Hamming distance between $z$ and a solution generated from~$p$.
    To this end, for all $j \in [n/2]$, let $q_j = (p_{2j - 1} + p_{2j})/2$.

    Recall that by the inequality of arithmetic and geometric mean, for all $a, b \in \R$, it holds that $ab \leq \big((a + b)/2\big)^2$.
    Using this inequality, the definition of~$E$, as well as that, for all $x \in \R$, it holds that $1 + x \leq \mathrm{e}^{x}$, we get
    \begin{align*}
        \Pr[E] &= \prod_{j \in [n/2]} \big(p_{2j - 1}p_{2j} + (1 - p_{2j - 1})(1 - p_{2j})\big)\\
        &\leq \prod_{j \in [n/2]} \big(q_j^2 + (1 - q_j)^2\big)\\
        &= \prod_{j \in [n/2]} \big(1 - 2q_j(1 - q_j)\big)\\
        &\leq \exp\Big(- \sum\nolimits_{j \in [n/2]} 2q_j(1 - q_j)\Big)\ .
    \end{align*}
    Further, since, for all $a, b \in \R$, it holds that $(a + b)^2 / 2 \leq a^2 + b^2$, it follows that
    \begin{align*}
        (a + b)\left(1 - \frac{a + b}{2}\right) &= a + b - \frac{(a + b)^2}{2}\\
        &\geq a + b - (a^2 + b^2)\\
        &= a(1 - a) + b(1 - b)\ .
    \end{align*}
    Using this inequality and the definition of~$q$, we further estimate
    \begin{align*}
        \Pr[E] &\leq \exp\bigg(-\sum_{j \in [n/2]} \big(p_{2j - 1}(1 - p_{2j - 1}) + p_{2j}(1 - p_{2j})\big)\bigg)\\
        &= \exp\Big(-\sum\nolimits_{i \in [n]} p_i(1 - p_i)\Big)\ .
    \end{align*}
    In addition, by the definition of~$z$, for all $i \in [n]$, a case distinction with respect to whether $p_i < 1/2$ or not shows that $p_i(1 - p_i) \geq |p_i - z_i|/2$.
    Thus, we get the estimate
    \begin{align*}
        \Pr[E] &\leq \exp\Big(-\sum\nolimits_{i \in [n]} \frac{|p_i - z_i|}{2}\Big)\\
        &= \exp\left(-\frac{\|p - z\|_1}{2}\right)\ .
    \end{align*}
    By the assumption $\Pr[E] \geq n^{-k}$ and \cref{eq:poisson_binomial_expectation}, it follows that
    \begin{align*}
        2 k \ln(n) \geq \|p - z\|_1 = \E[X] \ .
    \end{align*}

    Last, by \Cref{thm:chernoff_larger}, for all $\delta \ge 1$, it holds that
    \begin{align*}
        \Pr\big[X \geq (1 + \delta) \cdot 2 k \ln(n)\big] & \leq \exp\left(- \frac 23 \delta k \ln n \right)\\
				& = n^{- \frac 23 \delta k}.
    \end{align*}
		By using that $2\delta \ge (1+\delta)$ when $\delta \ge 1$ and renaming $\gamma \coloneqq 4\delta k$, we obtain the claim
		$\Pr[X \ge \gamma \ln(n)] \le n^{-\gamma/6}$ for all $\gamma \ge 4k$.
\end{proof}

\Cref{thm:number_of_optima_univariate} provides a strong connection between the probability of the event~$E$ that a sample~$x$ is an optimum of \ubm and the rounded frequency vector~$z$.
If it is somewhat likely for~$x$ to be optimal, that is, if there is a $k = O(1)$ such that $\Pr[E] \geq n^{-k}$, then~$x$ differs from~$z$, with high probability, in at most $4k\ln(n)$ bits.
Then each of the next $m \in \N^+$ bit strings created by~$p$, with probability at least $1 - m n^{-2k/3}$, differs from~$z$ in at most $4k\ln(n)$ positions.
These are at most $2k\ln(n)$ correct blocks, leading to at most $2^{2k\ln(n)} = n^{k\ln(4)}$ different optima, which is a polynomial independent of~$m$.
Thus, increasing the number of samples---although this decreases the probability of this line of argument to hold---does not increase the maximum number of different optima potentially created.
Ultimately, the number of different optima created over a polynomial number of iterations (in each of which~$p$ can even be chosen adversarially) with polynomially many samples in each iteration is still only polynomial, with high probability.

For the case that there is no~$k$ such that $\Pr[E] \geq n^{-k}$ or for the case that there is only a $k = \omega(1)$, then the probability to create a single optimal solution is already too small to be considered good, and the probabilistic model of the univariate EDA does not reflect the optima of \ubm well at all.

\section{Conclusion}
\label{sec:conclusion}

We proposed the \ubm benchmark as a test problem to see how well EDAs can develop a probabilistic model that copes with several different good solutions. We showed that MIMIC efficiently generates a probabilistic model that behaves very similarly to an ideal model. Since \ubm exhibits an exponential number of optima, this suggests that MIMIC is capable of implicitly storing a large range of different solutions in its model. Our experiments show that the model that MIMIC generates over time
\begin{itemize}
    \item has a permutation and border probabilities (almost) as in an ideal model, that the model
    \item does not create duplicate optimal solutions with increasing input size, and that it
    \item samples optima in each iteration with a probability that is close to the theoretical optimum of $1/\sqrt{\mathrm{e}}$.
\end{itemize}
Looking at sample data about the probabilistic model further suggests that the model is built such that it can generate an exponential number of optima. This is impressive, since this model is generated in a reasonable amount of time. We note that we used the plain MIMIC as found in the first paper proposing it~\cite{BonetIV96MIMIC} without any modifications.

In contrast, we show via mathematical means that no univariate model can come close to the advantages of these bivariate models. Whenever a univariate model is good enough to sample an optimum of \ubm with probability at least $n^{-k}$, then with very high probability all its samples lie in a Hamming ball of radius $4k\ln n$. It thus has enormous difficulties to sample most of the optima, which all lie outside this Hamming ball.

For future research, it is interesting to see if MIMIC also builds good models on more complicated functions with multiple optima, such as vertex cover on bipartite graphs. Further, since MIMIC has a very restricted type of bivariate model (namely, a path), considering other bivariate EDAs with a greater range of models, such as the bivariate marginal distribution algorithm (\cite{PelikanM99BMDA}; working on trees), would provide insights into whether the restriction of MIMIC's model to a path is a hindrance or not.

\section*{Acknowledgments}
This work was supported by COST action CA15140, by a public grant as part of the Investissement d'avenir project, reference ANR-11-LABX-0056-LMH, LabEx LMH, as well as by the Paris Île-de-France Region via the DIM RFSI AlgoSelect project and via the European Union's Horizon 2020 research and innovation program under the Marie Skłodowska-Curie grant agreement No. 945298-ParisRegionFP.

\bibliographystyle{elsarticle-num-names}
\balance
\bibliography{references}

\begin{thebibliography}{33}
\expandafter\ifx\csname natexlab\endcsname\relax\def\natexlab#1{#1}\fi
\providecommand{\url}[1]{\texttt{#1}}
\providecommand{\href}[2]{#2}
\providecommand{\path}[1]{#1}
\providecommand{\DOIprefix}{doi:}
\providecommand{\ArXivprefix}{arXiv:}
\providecommand{\URLprefix}{URL: }
\providecommand{\Pubmedprefix}{pmid:}
\providecommand{\doi}[1]{\href{http://dx.doi.org/#1}{\path{#1}}}
\providecommand{\Pubmed}[1]{\href{pmid:#1}{\path{#1}}}
\providecommand{\bibinfo}[2]{#2}
\ifx\xfnm\relax \def\xfnm[#1]{\unskip,\space#1}\fi
%Type = Article
\bibitem[{Belda et~al.(2007)Belda, Madurga, Tarragó, Llorà, and
  Giralt}]{BeldaMTLG07MulitmodalSearch}
\bibinfo{author}{I.~Belda}, \bibinfo{author}{S.~Madurga},
  \bibinfo{author}{T.~Tarragó}, \bibinfo{author}{X.~Llorà},
  \bibinfo{author}{E.~Giralt},
\newblock \bibinfo{title}{Evolutionary computation and multimodal search: {A}
  good combination to tackle molecular diversity in the field of peptide
  design},
\newblock \bibinfo{journal}{Molecular Diversity} \bibinfo{volume}{11}
  (\bibinfo{year}{2007}) \bibinfo{pages}{7--21}.
  \DOIprefix\doi{10.1007/s11030-006-9053-1}.
%Type = Article
\bibitem[{Hocaoǧlu and Sanderson(1997)}]{HocaogluS97MultimodalOptimization}
\bibinfo{author}{C.~Hocaoǧlu}, \bibinfo{author}{A.~C. Sanderson},
\newblock \bibinfo{title}{Multimodal function optimization using minimal
  representation size clustering and its application to planning multipaths},
\newblock \bibinfo{journal}{Evolutionary Computation} \bibinfo{volume}{5}
  (\bibinfo{year}{1997}) \bibinfo{pages}{81--104}.
  \DOIprefix\doi{10.1162/evco.1997.5.1.81}.
%Type = Inproceedings
\bibitem[{Singh and Deb(2006)}]{SinghD06NichingMethods}
\bibinfo{author}{G.~Singh}, \bibinfo{author}{K.~Deb},
\newblock \bibinfo{title}{Comparison of multi-modal optimization algorithms
  based on evolutionary algorithms},
\newblock in: \bibinfo{booktitle}{Proc.~of GECCO'06}, \bibinfo{year}{2006}, pp.
  \bibinfo{pages}{1305--1312}. \DOIprefix\doi{10.1145/1143997.1144200}.
%Type = Phdthesis
\bibitem[{De~Jong(1975)}]{DeJong75GeneticDrift}
\bibinfo{author}{K.~A. De~Jong}, \bibinfo{title}{An analysis of the behavior of
  a class of genetic adaptive systems}, Ph.D. thesis, University of Michigan,
  \bibinfo{address}{USA}, \bibinfo{year}{1975}.
%Type = Phdthesis
\bibitem[{Mahfoud(1996)}]{Mahfoud96Niching}
\bibinfo{author}{S.~W. Mahfoud}, \bibinfo{title}{Niching methods for genetic
  algorithms}, Ph.D. thesis, University of Illinois at Urbana-Champaign,
  \bibinfo{address}{USA}, \bibinfo{year}{1996}.
%Type = Inproceedings
\bibitem[{Miller and Shaw(1996)}]{MillerS96Niching}
\bibinfo{author}{B.~L. Miller}, \bibinfo{author}{M.~J. Shaw},
\newblock \bibinfo{title}{Genetic algorithms with dynamic niche sharing for
  multimodal function optimization},
\newblock in: \bibinfo{booktitle}{Proc.~of CEC'96}, \bibinfo{year}{1996}, pp.
  \bibinfo{pages}{786--791}. \DOIprefix\doi{10.1109/ICEC.1996.542701}.
%Type = Incollection
\bibitem[{Pelikan et~al.(2015)Pelikan, Hauschild, and
  Lobo}]{PelikanHL15SurveyOnEDAs}
\bibinfo{author}{M.~Pelikan}, \bibinfo{author}{M.~Hauschild},
  \bibinfo{author}{F.~G. Lobo},
\newblock \bibinfo{title}{Estimation of distribution algorithms},
\newblock in: \bibinfo{booktitle}{Springer Handbook of Computational
  Intelligence}, \bibinfo{publisher}{Springer}, \bibinfo{year}{2015}, pp.
  \bibinfo{pages}{899--928}. \DOIprefix\doi{10.1007/978-3-662-43505-2\_45}.
%Type = Book
\bibitem[{Larrañaga and Lozano(2002)}]{LarranagaL02EDAs}
\bibinfo{author}{P.~Larrañaga}, \bibinfo{author}{J.~A. Lozano},
  \bibinfo{title}{Estimation of distribution algorithms: a new tool for
  evolutionary computation}, \bibinfo{publisher}{Springer},
  \bibinfo{year}{2002}. \DOIprefix\doi{10.1007/978-1-4615-1539-5}.
%Type = Inproceedings
\bibitem[{Pelikan and Goldberg(2003)}]{PelikanG03hBOAperformsWell}
\bibinfo{author}{M.~Pelikan}, \bibinfo{author}{D.~E. Goldberg},
\newblock \bibinfo{title}{Hierarchical {BOA} solves {Ising} spin glasses and
  {MAXSAT}},
\newblock in: \bibinfo{booktitle}{Proc.~of GECCO'03}, \bibinfo{year}{2003}, pp.
  \bibinfo{pages}{1271--1282}. \DOIprefix\doi{10.1007/3-540-45110-2\_3}.
%Type = Article
\bibitem[{Peña et~al.(2005)Peña, Lozano, and
  Larrañaga}]{PenaLL05EDAsFindingMultipleOptima}
\bibinfo{author}{J.~Peña}, \bibinfo{author}{J.~Lozano},
  \bibinfo{author}{P.~Larrañaga},
\newblock \bibinfo{title}{Globally multimodal problem optimization via an
  estimation of distribution algorithm based on unsupervised learning of
  {Bayesian} networks},
\newblock \bibinfo{journal}{Evolutionary Computation} \bibinfo{volume}{13}
  (\bibinfo{year}{2005}) \bibinfo{pages}{43--66}.
  \DOIprefix\doi{10.1162/1063656053583432}.
%Type = Inproceedings
\bibitem[{Chuang and Hsu(2010)}]{ChuangH10MultimodalEDA}
\bibinfo{author}{C.-Y. Chuang}, \bibinfo{author}{W.-L. Hsu},
\newblock \bibinfo{title}{Multivariate multi-model approach for globally
  multimodal problems},
\newblock in: \bibinfo{booktitle}{Proc.~of GECCO'10}, \bibinfo{year}{2010}, pp.
  \bibinfo{pages}{311--318}. \DOIprefix\doi{10.1145/1830483.1830544}.
%Type = Inproceedings
\bibitem[{Hauschild et~al.(2007)Hauschild, Pelikan, Lima, and
  Sastry}]{HauschildPLS07hBOAprobabilisticModelAnalysis}
\bibinfo{author}{M.~Hauschild}, \bibinfo{author}{M.~Pelikan},
  \bibinfo{author}{C.~F. Lima}, \bibinfo{author}{K.~Sastry},
\newblock \bibinfo{title}{Analyzing probabilistic models in hierarchical {BOA}
  on traps and spin glasses},
\newblock in: \bibinfo{booktitle}{Proc.~of GECCO'07}, \bibinfo{year}{2007}, pp.
  \bibinfo{pages}{523--530}. \DOIprefix\doi{10.1145/1276958.1277070}.
%Type = Article
\bibitem[{Echegoyen et~al.(2012)Echegoyen, Mendiburu, Santana, and
  Lozano}]{EchegoyenMSL12AnalyzingProbabilisticModels}
\bibinfo{author}{C.~Echegoyen}, \bibinfo{author}{A.~Mendiburu},
  \bibinfo{author}{R.~Santana}, \bibinfo{author}{J.~A. Lozano},
\newblock \bibinfo{title}{Toward understanding {EDAs} based on {Bayesian}
  networks through a quantitative analysis},
\newblock \bibinfo{journal}{IEEE Transactions on Evolutionary Computation}
  \bibinfo{volume}{16} (\bibinfo{year}{2012}) \bibinfo{pages}{173--189}.
  \DOIprefix\doi{10.1109/TEVC.2010.2102037}.
%Type = Inproceedings
\bibitem[{Etxeberria and Larrañaga(1999)}]{EtxeberriaL99EBNA}
\bibinfo{author}{R.~Etxeberria}, \bibinfo{author}{P.~Larrañaga},
\newblock \bibinfo{title}{Global optimization with {Bayesian} networks},
\newblock in: \bibinfo{booktitle}{Proc.~of CIMAF'99}, \bibinfo{year}{1999}, pp.
  \bibinfo{pages}{332--339}.
%Type = Misc
\bibitem[{Doerr and Krejca(2022)}]{OurCode}
\bibinfo{author}{B.~Doerr}, \bibinfo{author}{M.~S. Krejca},
  \bibinfo{title}{Code repository of this paper},
  \bibinfo{howpublished}{\url{https://github.com/TheMor/TheMor-MIMIC_Multiple_Optima}},
  \bibinfo{year}{2022}.
%Type = Inproceedings
\bibitem[{Bonet et~al.(1996)Bonet, Jr., and Viola}]{BonetIV96MIMIC}
\bibinfo{author}{J.~S.~D. Bonet}, \bibinfo{author}{C.~L.~I. Jr.},
  \bibinfo{author}{P.~A. Viola},
\newblock \bibinfo{title}{{MIMIC:} finding optima by estimating probability
  densities},
\newblock in: \bibinfo{booktitle}{Proc.~of NIPS'96}, \bibinfo{year}{1996}, pp.
  \bibinfo{pages}{424--430}.
%Type = Inproceedings
\bibitem[{Doerr and Krejca(2020)}]{DoerrK20MIMICforGECCO}
\bibinfo{author}{B.~Doerr}, \bibinfo{author}{M.~S. Krejca},
\newblock \bibinfo{title}{Bivariate estimation-of-distribution algorithms can
  find an exponential number of optima},
\newblock in: \bibinfo{booktitle}{Proc.~of GECCO'20}, \bibinfo{year}{2020}, pp.
  \bibinfo{pages}{796--804}. \DOIprefix\doi{10.1145/3377930.3390177}.
%Type = Incollection
\bibitem[{Henrion(1988)}]{Henrion1988ProbabilisticLogicSampling}
\bibinfo{author}{M.~Henrion},
\newblock \bibinfo{title}{Propagating uncertainty in {Bayesian} networks by
  probabilistic logic sampling},
\newblock in: \bibinfo{booktitle}{Uncertainty in Artificial Intelligence},
  volume~\bibinfo{volume}{5} of \textit{\bibinfo{series}{Machine Intelligence
  and Pattern Recognition}}, \bibinfo{publisher}{North-Holland},
  \bibinfo{year}{1988}, pp. \bibinfo{pages}{149--163}.
  \DOIprefix\doi{10.1016/B978-0-444-70396-5.50019-4}.
%Type = Book
\bibitem[{Koller and Friedman(2009)}]{KollerF09GraphicalModels}
\bibinfo{author}{D.~Koller}, \bibinfo{author}{N.~Friedman},
  \bibinfo{title}{Probabilistic graphical models: principles and techniques},
  \bibinfo{publisher}{The MIT Press}, \bibinfo{year}{2009}.
%Type = Article
\bibitem[{Harik et~al.(1999)Harik, Lobo, and Goldberg}]{HarikLG99cGA}
\bibinfo{author}{G.~R. Harik}, \bibinfo{author}{F.~G. Lobo},
  \bibinfo{author}{D.~E. Goldberg},
\newblock \bibinfo{title}{The compact genetic algorithm},
\newblock \bibinfo{journal}{IEEE Transactions on Evolutionary Computation}
  \bibinfo{volume}{3} (\bibinfo{year}{1999}) \bibinfo{pages}{287--297}.
  \DOIprefix\doi{10.1109/4235.797971}.
%Type = Inproceedings
\bibitem[{Mühlenbein and Paaß(1996)}]{MuehlenbeinP96UMDA}
\bibinfo{author}{H.~Mühlenbein}, \bibinfo{author}{G.~Paaß},
\newblock \bibinfo{title}{From recombination of genes to the estimation of
  distributions {I.} {Binary} parameters},
\newblock in: \bibinfo{booktitle}{Proc.~of PPSN~IV}, \bibinfo{year}{1996}, pp.
  \bibinfo{pages}{178--187}. \DOIprefix\doi{10.1007/3-540-61723-X_982}.
%Type = Inproceedings
\bibitem[{Friedrich et~al.(2016)Friedrich, Kötzing, and
  Krejca}]{FriedrichKK16BalancedStable}
\bibinfo{author}{T.~Friedrich}, \bibinfo{author}{T.~Kötzing},
  \bibinfo{author}{M.~S. Krejca},
\newblock \bibinfo{title}{{EDA}s cannot be balanced and stable},
\newblock in: \bibinfo{booktitle}{Proc.~of GECCO'16}, \bibinfo{year}{2016}, pp.
  \bibinfo{pages}{1139--1146}. \DOIprefix\doi{10.1145/2908812.2908895}.
%Type = Article
\bibitem[{Doerr and Zheng(2020)}]{DoerrZ20GeneticDrift}
\bibinfo{author}{B.~Doerr}, \bibinfo{author}{W.~Zheng},
\newblock \bibinfo{title}{Sharp bounds for genetic drift in estimation of
  distribution algorithms},
\newblock \bibinfo{journal}{{IEEE} Transactions on Evolutionary Computation}
  \bibinfo{volume}{24} (\bibinfo{year}{2020}) \bibinfo{pages}{1140--1149}.
  \DOIprefix\doi{10.1109/TEVC.2020.2987361}.
%Type = Incollection
\bibitem[{Krejca and Witt(2020)}]{KrejcaW20EDAsurvey}
\bibinfo{author}{M.~S. Krejca}, \bibinfo{author}{C.~Witt},
\newblock \bibinfo{title}{Theory of estimation-of-distribution algorithms},
\newblock in: \bibinfo{booktitle}{Theory of Evolutionary Computation: Recent
  Developments in Discrete Optimization}, \bibinfo{publisher}{Springer
  International Publishing}, \bibinfo{year}{2020}, pp.
  \bibinfo{pages}{405--442}. \URLprefix \url{https://arxiv.org/abs/1806.05392}.
  \DOIprefix\doi{10.1007/978-3-030-29414-4}.
%Type = Incollection
\bibitem[{Pelikan and Mühlenbein(1999)}]{PelikanM99BMDA}
\bibinfo{author}{M.~Pelikan}, \bibinfo{author}{H.~Mühlenbein},
\newblock \bibinfo{title}{The bivariate marginal distribution algorithm},
\newblock in: \bibinfo{booktitle}{Advances in Soft Computing},
  \bibinfo{publisher}{Springer}, \bibinfo{year}{1999}, pp.
  \bibinfo{pages}{521--535}. \DOIprefix\doi{10.1007/978-1-4471-0819-1\_39}.
%Type = Inproceedings
\bibitem[{Lehre and Nguyen(2019)}]{LehreN19UMDAonDeception}
\bibinfo{author}{P.~K. Lehre}, \bibinfo{author}{P.~T.~H. Nguyen},
\newblock \bibinfo{title}{On the limitations of the univariate marginal
  distribution algorithm to deception and where bivariate {EDAs} might help},
\newblock in: \bibinfo{booktitle}{Proc.~of FOGA'19}, \bibinfo{year}{2019}, pp.
  \bibinfo{pages}{154--168}. \DOIprefix\doi{10.1145/3299904.3340316}.
%Type = Article
\bibitem[{Doerr and Krejca(2021)}]{DoerrK21UMDAwellOnDeception}
\bibinfo{author}{B.~Doerr}, \bibinfo{author}{M.~S. Krejca},
\newblock \bibinfo{title}{The univariate marginal distribution algorithm copes
  well with deception and epistasis},
\newblock \bibinfo{journal}{Evolutionary Computation} \bibinfo{volume}{29}
  (\bibinfo{year}{2021}) \bibinfo{pages}{543--563}.
  \DOIprefix\doi{10.1162/evco\_a\_00293}.
%Type = Techreport
\bibitem[{Harik(1999)}]{Harik99ECGA}
\bibinfo{author}{G.~Harik}, \bibinfo{title}{Linkage Learning via Probabilistic
  Modeling in the {ECGA}}, \bibinfo{type}{Technical Report}
  \bibinfo{number}{99010}, University of Illinois Urbana-Champaign,
  \bibinfo{address}{Urbana, IL, USA}, \bibinfo{year}{1999}.
%Type = Article
\bibitem[{Mühlenbein and Mahnig(1999)}]{MuehlenbeinM1999FDA}
\bibinfo{author}{H.~Mühlenbein}, \bibinfo{author}{T.~Mahnig},
\newblock \bibinfo{title}{{FDA -- A} scalable evolutionary algorithm for the
  optimization of additively decomposed functions},
\newblock \bibinfo{journal}{Evolutionary Computation} \bibinfo{volume}{7}
  (\bibinfo{year}{1999}) \bibinfo{pages}{353--376}.
  \DOIprefix\doi{10.1162/evco.1999.7.4.353}.
%Type = Inproceedings
\bibitem[{Pelikan et~al.(1999)Pelikan, Goldberg, and
  Cant{\'u}-Paz}]{PelikanGCP1999BOA}
\bibinfo{author}{M.~Pelikan}, \bibinfo{author}{D.~E. Goldberg},
  \bibinfo{author}{E.~Cant{\'u}-Paz},
\newblock \bibinfo{title}{{BOA}: {The Bayesian} optimization algorithm},
\newblock in: \bibinfo{booktitle}{Proc.~of GECCO'99}, \bibinfo{year}{1999}, pp.
  \bibinfo{pages}{525--532}.
%Type = Inproceedings
\bibitem[{Pelikan and Goldberg(2001)}]{PelikanG01hBOA}
\bibinfo{author}{M.~Pelikan}, \bibinfo{author}{D.~E. Goldberg},
\newblock \bibinfo{title}{Escaping hierarchical traps with competent genetic
  algorithms},
\newblock in: \bibinfo{booktitle}{Proc.~of GECCO'01}, \bibinfo{year}{2001}, pp.
  \bibinfo{pages}{511--518}.
%Type = Article
\bibitem[{Hoeffding(1963)}]{Hoeffding63}
\bibinfo{author}{W.~Hoeffding},
\newblock \bibinfo{title}{Probability inequalities for sums of bounded random
  variables},
\newblock \bibinfo{journal}{Journal of~the American Statistical Association}
  \bibinfo{volume}{58} (\bibinfo{year}{1963}) \bibinfo{pages}{13--30}.
%Type = Incollection
\bibitem[{Doerr(2020)}]{Doerr20bookchapter}
\bibinfo{author}{B.~Doerr},
\newblock \bibinfo{title}{Probabilistic tools for the analysis of randomized
  optimization heuristics},
\newblock in: \bibinfo{booktitle}{Theory of Evolutionary Computation: Recent
  Developments in Discrete Optimization}, \bibinfo{publisher}{Springer
  International Publishing}, \bibinfo{year}{2020}, pp. \bibinfo{pages}{1--87}.
  \URLprefix \url{https://arxiv.org/abs/1801.06733}.

\end{thebibliography}

\end{document}